\documentclass{article}

\usepackage{arxiv}

\usepackage[utf8]{inputenc} 
\usepackage[T1]{fontenc}    
\usepackage{hyperref}       
\usepackage{url}            
\usepackage{booktabs}       
\usepackage{amsfonts}       
\usepackage{nicefrac}       
\usepackage{microtype}      
\usepackage{lipsum}		
\usepackage{graphicx}
\usepackage{natbib}
\usepackage{doi}

\usepackage{amsmath,amssymb,amsthm}
\usepackage{geometry}
\usepackage{hyperref}
\usepackage{bm}
\usepackage{authblk}

\usepackage{tikz}
\usetikzlibrary{arrows.meta,calc,decorations.pathmorphing}
\usepackage{xcolor}

\usepackage{authblk}

\usepackage{framed}

\newtheorem{theorem}{Theorem}
\newtheorem{lemma}{Lemma}

\theoremstyle{remark}
\newtheorem{remark}{Remark}

\title{
A Thermodynamic Theory of Learning I: \\ Irreversible Ensemble Transport and Epistemic Costs
}

\author{Daisuke Okanohara}
\affil{Preferred Networks, Inc.}
\begin{document}
\maketitle

\begingroup
\renewcommand\thefootnote{}
\footnotetext{
This paper is the first part of a planned series entitled
``A Thermodynamic Theory of Learning,'' which develops a
thermodynamic framework for finite-time learning dynamics.
}
\endgroup

\begin{abstract}
Learning systems acquire structured internal representations through training,
yet classical information-theoretic results assert that deterministic
transformations cannot increase information.
This apparent tension raises a fundamental question:
how can learning produce meaningful epistemic structure without violating
information-theoretic limits (e.g., data processing inequalities)?

In this work, we argue that this tension is resolved by recognizing learning as
an inherently finite-time and irreversible process.
We formulate learning at the ensemble level as a transport process in the space
of probability distributions over model configurations.
Within this framework, we introduce an epistemic free-energy functional that
balances objective-driven potential improvement against the loss of ensemble
diversity.

We show that reductions in epistemic free energy accumulated along a learning
trajectory admit a decomposition into reversible contributions associated with
changes in the objective landscape and irreversible contributions quantified by
entropy production.
Crucially, this entropy production reflects finite-time epistemic commitment and
cannot be eliminated by algorithmic optimization alone.

Building on this formulation, we derive an \emph{Epistemic Speed Limit} (ESL),
a finite-time inequality that lower-bounds the irreversible entropy production
required to realize a prescribed ensemble transformation.
The bound is geometric in nature, depending only on the transport distance
between initial and final distributions, and holds independently of the
specific learning algorithm.

Our results establish a thermodynamic theory of learning in which differences
between learning procedures arise not from access to different epistemic
resources, but from how efficiently they manage unavoidable irreversible cost.
This perspective reconciles abstraction with information-theoretic constraints
and highlights finite-time irreversibility as a central organizing principle of
learning dynamics.
\end{abstract}

\section{Introduction}

Learning systems acquire structured internal representations from data.
In modern machine learning, this process enables models to extract abstractions,
discover latent regularities, and generalize beyond observed examples.
At the same time, classical results in information theory assert that
deterministic transformations cannot increase information content.
This apparent tension raises a fundamental question:
how can learning give rise to meaningful structure and abstraction
without contradicting established information-theoretic principles?

A key observation of this work is that learning is inherently a finite-time process:
meaningful learning requires committing to epistemic structure within finite time~\cite{epiplexity2026}.
While asymptotic analyses often focus on the existence or properties of optimal solutions,
real learning systems must operate under strict constraints on time, data, and computation.
Under such constraints, learning dynamics are generally irreversible:
probability mass is compressed, hypotheses are discarded, and epistemic uncertainty is reduced
in ways that cannot be undone without additional cost.
Understanding the nature and consequences of this irreversibility
is therefore essential for a principled theory of learning.

In this work, we adopt a thermodynamic perspective on learning.
We describe the state of a learning system by a probability distribution
over model configurations and model learning as a transport process
in the space of such distributions.
Within this framework, we introduce an epistemic free-energy functional
that balances an objective-driven contribution—captured by a learning objective
such as empirical risk or negative log-likelihood—against epistemic uncertainty
over model configurations.
This formulation provides a unified language for quantifying
both progress in learning and the costs incurred along a learning trajectory.

A central quantity in our analysis is the free-energy reduction,
defined as a bookkeeping quantity that records the cumulative change
in epistemic free energy along a learning trajectory.
Importantly, this quantity specifies how much free energy is reduced,
but not how that reduction is realized dynamically.
To resolve this distinction, we analyze the decomposition of free-energy change
into reversible and irreversible contributions,
with the irreversible contribution corresponding to entropy production.

While the formalism developed here is grounded in thermodynamic concepts,
it is important to clarify the role played by entropy production in learning.
Entropy production should not be interpreted as directly determining
the final performance of a learning system.
Rather, it characterizes the irreversibility of the learning trajectory
under finite-time constraints.
In an idealized regime where learning dynamics can be approximated
as gradient flows of a time-independent free-energy functional
without external driving,
the total free-energy reduction along a learning trajectory
coincides with the total entropy production in this limit.
This regime provides a useful baseline in which the irreversible cost of learning
can be cleanly isolated.

Practical learning algorithms, however, rarely operate under such idealized conditions.
Common training practices—including time-dependent learning-rate schedules,
curriculum learning, and adaptive regularization—
effectively introduce external driving that modifies the free-energy landscape over time.
From this perspective, entropy production does not constrain
which configurations a learning system may ultimately reach,
but instead limits how efficiently a given distributional transformation
can be realized within finite time.
Excessive entropy production is therefore expected to manifest
not as degraded asymptotic performance,
but as instability, sensitivity to hyperparameters,
or poor reproducibility of learning outcomes.

Building on this framework, we derive an \emph{Epistemic Speed Limit} (ESL),
a finite-time inequality that lower-bounds the entropy production required
to transport probability mass between distributions over a given time interval.
Crucially, this bound is geometric in nature:
it depends only on the statistical distance between the initial and final distributions
and the duration of the learning process,
and holds independently of the specific learning algorithm
or form of external driving, including optimally controlled dynamics.
The ESL thus reveals a fundamental constraint on learning dynamics,
clarifying the minimal irreversible cost associated with committing
to epistemic structure under finite-time constraints.

\section{Related Work}

\paragraph{Epiplexity and structural information.}
The concept of epiplexity was recently introduced as a measure of the amount of epistemic structure available to a computationally bounded observer~\cite{epiplexity2026}.
This notion builds on a long tradition of research on minimum description length (MDL)~\cite{rissanen1978mdl,rissanen1989,grunwald2007mdl}
and resource-bounded variants of Kolmogorov complexity~\cite{li_vitanyi},
which characterize structure in terms of compressibility under computational constraints.
In this line of work, epiplexity is fundamentally a static, in-principle quantity:
it quantifies how much epistemic structure is present independently of how learning is actually carried out,
and is typically defined through optimal descriptions or representations.
In contrast, the present work is concerned not with the amount of structure that can in principle be extracted,
but with the irreversible costs incurred when committing to such structure along finite-time learning trajectories. The present work should be viewed as complementary to epiplexity, addressing not availability but realization.

\paragraph{Free energy and variational learning.}
Free-energy formulations play a central role in variational inference and Bayesian learning,
where learning is interpreted as the minimization of a variational free-energy objective~\cite{neal_hinton_1998,wainwright_jordan_2008,bishop2006}.
These approaches provide a continuous, distributional relaxation of discrete model selection
and MDL-based objectives, allowing learning processes to be described as trajectories in probability space.
The present work adopts this free-energy perspective not as a normative learning principle,
but as a dynamical formalism for analyzing learning as a finite-time transport process.
Here, free energy is not primarily interpreted as a measure of subjective belief or uncertainty,
but as a representational tool that enables the decomposition of learning progress
into objective-driven contributions and irreversible entropy production along learning trajectories.

\paragraph{Finite-time constraints and speed limits.}
Fundamental limits on the rate of state transformations have been extensively studied
in nonequilibrium thermodynamics, where speed limits relate achievable transformations
to entropy production and geometric properties of state space~\cite{shiraishi2018speed,ito2024geometric}.
A particularly clear class of results establishes universal lower bounds on dissipation
in terms of geometric distances between states~\cite{PRXspeedlimit}.
The present work can be viewed as an epistemic analogue of these thermodynamic speed limits,
with probability distributions over model configurations playing the role of physical states.
Closely related geometric ideas also arise in optimal transport theory,
which provides natural metrics on spaces of probability distributions~\cite{villani2008optimal}.
The Epistemic Speed Limit derived here applies these finite-time irreversibility principles
to learning dynamics, yielding bounds on the minimal irreversible cost
required to transform distributions over model configurations within finite time.

\paragraph{Learning procedure dependence.}
The strong dependence of learning outcomes on training procedures has long been recognized in practice.
Curriculum learning~\cite{bengio2009curriculum} and knowledge distillation~\cite{hinton2015distillation}
are prominent examples in which structured or gradual training procedures
lead to substantial improvements, even without additional data.
Despite their empirical success, principled theoretical accounts of why learning procedures matter remain limited.
From the perspective developed here, such methods can be interpreted as shaping learning trajectories
so as to improve the efficiency of probability transport in distribution space,
thereby managing the entropy production associated with finite-time learning.
Importantly, entropy production itself is not inherently detrimental:
in non-convex landscapes, stochastic fluctuations—and the associated entropy production—
may be essential for exploration and for escaping unfavorable basins.
The Epistemic Speed Limit instead highlights the cost of \emph{uncontrolled} entropy production,
arising from geometrically inefficient transport.
Within this framework, advanced training strategies can be understood
as allocating entropy production toward necessary exploration
while limiting avoidable irreversibility.

\section{Learning as an Irreversible Thermodynamic Process}
\label{sec:thermo}

In this section, we formulate learning as a finite-time \emph{irreversible}
process at the level of ensembles of model configurations.
The role of thermodynamic language here is explicitly \emph{descriptive} rather
than prescriptive: the quantities introduced below serve as bookkeeping devices
for analyzing ensemble-level learning dynamics.
They do not define a learning objective, nor do they prescribe an optimization
algorithm.

Our goal is to make precise in what sense learning is irreversible, what is lost
along a learning trajectory, and how such losses can be quantified independently
of any particular notion of optimality or generalization.
These considerations will form the foundation for the Epistemic Speed Limit
derived in Section~\ref{sec:esl}.

\subsection{Ensemble-Level Description of Learning}

We describe learning at the ensemble level by a probability distribution
$q_s(\theta)$ over a configuration space $\Theta$, indexed by a rescaled time
parameter $s\in[0,1]$.
An ensemble represents the set of model configurations that can be realized
under variations in initialization, data ordering, stochasticity, or other
training conditions.

This probabilistic description is not intended to represent subjective
uncertainty or Bayesian belief.
Rather, it provides a geometric representation of how learning redistributes
accessible configurations across parameter space as training progresses.

Viewed in this way, learning is generically irreversible.
As training proceeds, probability mass concentrates and regions of
configuration space are abandoned, so that the initial ensemble cannot be
recovered without incurring additional cost.
This ensemble-level irreversibility motivates the transport-based
formulation introduced in the following subsection.

\subsection{Learning as Distributional Transport}

We model learning as a continuous transport of probability distributions
$q_s(\theta)$ over model parameters, indexed by a rescaled time $s\in[0,1]$.
The evolution of $q_s$ is governed by the continuity equation
\begin{equation}
\partial_s q_s + \nabla\cdot(q_s v_s)=0,
\end{equation}
where $v_s(\theta)$ denotes the probability velocity field in parameter space.

As a geometric measure of irreversible ensemble transport,
we define the \emph{epistemic entropy production rate}\footnote{Throughout this paper, the term ``entropy production'' is used in an epistemic and transport-theoretic sense.
Specifically, it refers to the quadratic action associated with probability transport in Wasserstein space.
This notion should be distinguished from the conventional thermodynamic entropy production defined in stochastic
thermodynamics via heat exchange with an external bath.} as
\begin{equation}
\sigma_s
:=
\int q_s(\theta)\,\|v_s(\theta)\|^2\,d\theta,
\qquad
\sigma_s \ge 0.
\end{equation}
The quantity $\Sigma_{0:1}$ will serve as a geometric measure of irreversible
probability transport and forms the basic cost functional underlying the
Epistemic Speed Limit derived below.
\begin{equation}
\Sigma_{0:1}
\;:=\;
\int_0^1 \sigma_s\,ds.
\end{equation}

This quantity coincides with the Benamou--Brenier action associated with the
$2$-Wasserstein distance and measures the minimal irreversible cost required
to transport probability mass between ensemble states.
Throughout this paper, entropy production is understood in this epistemic,
transport-theoretic sense.

\subsection{Fokker--Planck Dynamics as a Representative Model}

In this section, we focus on an idealized regime in which the effective noise temperature $T$ is treated as constant.
This assumption allows a clean dissipation identity and highlights the geometric nature of irreversible ensemble transport.
In practical learning algorithms such as SGD, the effective noise scale may vary over time due to learning-rate schedules
or batch-size changes; such time-dependent driving generically introduces additional terms in the free-energy balance.
However, the Epistemic Speed Limit derived here concerns the minimal geometric cost of probability transport and does not
rely on the detailed temporal profile of $T$.

We use Fokker--Planck dynamics as a representative and analytically tractable model of irreversible ensemble learning.
This choice is made for clarity rather than realism.
The Fokker--Planck equation provides a canonical setting in which irreversibility,
entropy production, and transport geometry can be computed explicitly.

Concretely, we consider a continuous-time approximation of stochastic gradient
descent.
Let $\theta_s$ evolve according to the Langevin equation
\begin{equation}
d\theta_s
=
-\nabla\Phi(\theta_s)\,ds
+
\sqrt{2T}\,dW_s,
\end{equation}
where $\Phi(\theta)$ is a learning objective, $T>0$ sets the scale of stochasticity,
and $W_s$ denotes standard Brownian motion.

This dynamics can be written in continuity form as
\begin{equation}
\partial_s q_s + \nabla \cdot (q_s v_s) = 0,
\end{equation}
with the associated probability-flow velocity field
\begin{equation}
v_s = -\nabla \bigl(\Phi + T \log q_s \bigr).
\end{equation}

Substituting this expression for $v_s$ into the continuity equation
recovers the Fokker--Planck equation below.
\begin{equation}
\label{eq:fp}
\partial_s q_s
=
\nabla\cdot\!\left(q_s \nabla \Phi(\theta)\right)
+
T\,\Delta q_s.
\end{equation}
This dynamics can be interpreted as the Wasserstein gradient flow of the
epistemic free-energy functional introduced below
\cite{jordan1998variational,ambrosio2008gradient}.

We emphasize that Gaussian noise should be understood as a convenient
approximation rather than a literal model of stochastic gradient noise.
The results derived below rely not on Gaussianity itself, but on the more general
facts that learning induces irreversible ensemble transport and that such
transport incurs a geometric cost over finite time.

\subsection{Epistemic Free Energy and Entropy Production}

To analyze ensemble-level learning dynamics, we introduce the epistemic free
energy
\begin{equation}
\label{eq:free_energy}
\mathcal F[q]
:=
\mathbb E_q[\Phi]
-
T H[q],
\qquad
H[q]:=-\int q(\theta)\log q(\theta)\,d\theta.
\end{equation}
The first term measures the ensemble-averaged objective value, while the second
term penalizes concentration of probability mass.
The free energy thus balances objective improvement against loss of ensemble
diversity.

Under Fokker--Planck dynamics, the epistemic free energy $\mathcal F[q_s]$
satisfies a dissipation identity relating its rate of change
to the entropy production rate $\sigma_s$ defined above:
\begin{equation}
\label{eq:fe_dissipation}
\frac{d}{ds}\mathcal F[q_s]
=
-\sigma_s,
\qquad
\sigma_s\ge 0,
\end{equation}
where $\sigma_s$ denotes the instantaneous entropy production rate.
A complete derivation of the dissipation identity,
including all regularity assumptions and integration-by-parts steps,
is provided in Appendix~\ref{app:free_energy_monotonicity}.
Integrating over $s\in[0,1]$ yields
\begin{equation}
\label{eq:total_dissipation}
\mathcal F[q_0]-\mathcal F[q_1]
=
\Sigma_{0:1},
\qquad
\Sigma_{0:1}:=\int_0^1\sigma_s\,ds.
\end{equation}

Equation~\eqref{eq:total_dissipation} expresses a fundamental irreversibility. 
In this special dynamical setting, the decrease in free energy accumulated over
the learning trajectory is exactly accounted for by the entropy production.
More generally, however, entropy production and free-energy change are distinct
quantities: the former depends on the entire learning path, while the latter
depends only on the endpoints.
This identity is dynamical and does not depend on any notion of optimality.

\subsection{Free-Energy Decomposition}

Independently of the dynamics, the free-energy difference between two ensemble
states admits the algebraic decomposition
\begin{equation}
\label{eq:fe_decomposition}
\mathcal F[q_0]-\mathcal F[q_1]
=
\bigl(\mathbb E_{q_0}[\Phi]-\mathbb E_{q_1}[\Phi]\bigr)
+
T\bigl(H[q_1]-H[q_0]\bigr).
\end{equation}
The first term represents the change in the ensemble-averaged objective, while
the second term accounts for the change in entropy along the learning trajectory
(it is negative when entropy decreases).

This decomposition is purely definitional.
In particular, it does not imply that objective improvement and entropy reduction
are causally related.
The role of the free energy is to provide a unified bookkeeping device in which
both contributions appear on equal footing.
Constraints on learning dynamics therefore apply to the \emph{total} free-energy
change, rather than to its individual components.

\subsection{Gaussian Noise as a Representative Approximation}
Gaussian noise is assumed throughout this work for analytical clarity, as it
leads to a Fokker–Planck description with a quadratic entropy production rate
and a natural Wasserstein-2 geometry.
However, Gaussianity is not essential to the conceptual structure of our
results.
Empirical studies indicate that stochastic gradient noise is often heavy-tailed,
in which case the ensemble dynamics are more appropriately described by
Lévy-type or fractional diffusion processes.
While such dynamics modify the precise form of the transport cost and the
associated geometry, they preserve the fundamental features emphasized here:
learning induces irreversible probability transport, the free energy decreases monotonically, and finite-time ensemble transformations incur
unavoidable geometric costs.
We therefore view the Fokker–Planck formulation as a representative model rather
than a literal description of stochastic gradient dynamics.

\subsection{Transport Geometry and Irreversible Cost}

We now relate the entropy production 
to the geometry of probability transport.

Let $(q_s, j_s)$ be a learning trajectory satisfying the continuity equation
\begin{equation}
\partial_s q_s + \nabla\cdot j_s = 0,
\end{equation}
with fixed endpoints $q_0$ and $q_1$.
The associated probability velocity field is defined by
$j_s = q_s v_s$.

A fundamental geometric lower bound on the irreversible cost of such transport
is provided by optimal transport theory.
Specifically, the squared $2$-Wasserstein distance $W_2(q_0,q_1)^2$ defines
a metric on the space of probability distributions over $\Theta$ with finite
second moments and admits the variational representation
\begin{equation}
W_2(q_0,q_1)^2
=
\inf_{(q_s,v_s)}
\int_0^1\!\int q_s(\theta)\,\|v_s(\theta)\|^2\,d\theta\,ds,
\end{equation}
where the infimum is taken over all paths satisfying the continuity equation
with fixed endpoints.
This quantity can be interpreted as the minimal quadratic cost required to
transport probability mass from $q_0$ to $q_1$; see~\cite{villani2008optimal}
for details.

Since the actual learning trajectory $(q_s, v_s)$ is one admissible path in
this variational problem, the entropy production accumulated along learning
is bounded from below by $W_2(q_0,q_1)^2$.
This lower bound expresses a purely geometric and algorithm-independent
constraint on finite-time learning dynamics.

\section{Epistemic Speed Limit}
\label{sec:esl}

Learning is represented as a trajectory in the space of probability
distributions over model configurations, evolving from an initial ensemble
$q_0$ to a terminal ensemble $q_1$ over a finite time horizon.
As discussed in the previous section, finite-time learning dynamics
necessarily incur irreversible costs associated with probability transport.

In this section, we formalize this constraint for a specific and analytically
tractable class of learning dynamics.
Our result is an epistemic analogue of thermodynamic speed limits, which bound
dissipation in terms of geometric distances between states
in nonequilibrium systems~\cite{PRXspeedlimit}.

Specifically, we derive a finite-time lower bound on the entropy production
required to realize a prescribed ensemble transformation under
Fokker--Planck dynamics.
We refer to this bound as the \emph{Epistemic Speed Limit} (ESL).

\begin{theorem}[Epistemic Speed Limit under Fokker--Planck Dynamics]
\label{thm:esl}
Let $q_s(\theta)$ evolve according to the Fokker--Planck equation
\[
\partial_s q_s
=
\nabla\cdot\!\left(q_s\nabla\Phi\right)
+
T\Delta q_s,
\qquad s\in[0,1],
\]
from an initial distribution $q_0$ to a terminal distribution $q_1$.

Let $\Sigma_{0:1}:=\int_0^1\sigma_s\,ds$ denote the entropy production accumulated
along the learning trajectory, where $\sigma_s$ is defined in
Section~\ref{sec:thermo}.
Then the following statements hold:
\begin{enumerate}
\item \textbf{(Free-energy dissipation)}
\begin{equation}
\label{eq:esl_free_energy}
\mathcal F[q_0]-\mathcal F[q_1]
=
\Sigma_{0:1}.
\end{equation}

\item \textbf{(Epistemic Speed Limit)}
\begin{equation}
\label{eq:esl_speed_limit}
\Sigma_{0:1}
\;\ge\;
W_2(q_0,q_1)^2.
\end{equation}

\item \textbf{(Consequence for objective improvement)}
\begin{equation}
\label{eq:esl_objective}
\mathbb E_{q_0}[\Phi]-\mathbb E_{q_1}[\Phi]
\;\ge\;
W_2(q_0,q_1)^2
+
T\bigl(H[q_0]-H[q_1]\bigr).
\end{equation}
\end{enumerate}

Moreover, the bound in \eqref{eq:esl_speed_limit} is tight: equality is achieved for
a constant-speed Wasserstein geodesic between $q_0$ and $q_1$,
although such a trajectory does not necessarily coincide with the Fokker--Planck dynamics
for a fixed objective $\Phi$.
\end{theorem}

\begin{proof}
The free-energy dissipation identity follows from the standard calculation
for the Wasserstein gradient flow of $\mathcal F$.

For the speed limit, recall that the squared Wasserstein distance
$W_2(q_0,q_1)^2$ is defined as the minimum quadratic transport cost over all
admissible probability flows connecting $q_0$ and $q_1$.
Since the actual learning trajectory $(q_s,v_s)$ is one such admissible flow,
the entropy production accumulated along learning must satisfy
\[
\Sigma_{0:1}
=
\int_0^1\!\int q_s(\theta)\,\|v_s(\theta)\|^2\,d\theta\,ds
\;\ge\;
W_2(q_0,q_1)^2.
\]
The final inequality follows by combining the Epistemic Speed Limit with
the algebraic decomposition of the free energy.
\end{proof}

\begin{remark}[Finite-time scaling of the Epistemic Speed Limit]
\label{rem:esl-finite-time}
Throughout this work, learning trajectories are parametrized by a normalized
time variable $s\in[0,1]$.
To relate the Epistemic Speed Limit to a physical training time
$\tau\in[0,\mathcal{T}]$, consider the time reparametrization
$\tau = \mathcal{T}s$.
Under this change of variables, the continuity equation remains invariant,
while the probability velocity field rescales as
$v_\tau = (1/\mathcal{T})\,v_s$.

As a consequence, the entropy production accumulated over physical time satisfies
\[
\Sigma_{0:\mathcal{T}}
=
\int_0^{\mathcal{T}}\!\int q_\tau(\theta)\,\|v_\tau(\theta)\|^2\,d\theta\,d\tau
=
\frac{1}{\mathcal{T}}\,\Sigma_{0:1}.
\]
This scaling shows that, for a fixed ensemble transformation,
the minimal irreversible cost diverges as the available physical training
time $\mathcal{T}$ decreases.

Combining this scaling with the free-energy dissipation identity
$\mathcal F[q_0]-\mathcal F[q_1]=\Sigma_{0:\mathcal T}$
yields the corresponding finite-time bound on free-energy reduction:
\begin{equation}
\label{eq:fe_drop_speed_limit}
\mathcal F[q_0]-\mathcal F[q_1]
\;\ge\;
\frac{1}{\mathcal T}\,W_2(q_0,q_1)^2.
\end{equation}

This inequality should be interpreted as a geometric lower bound determined by the available learning time.
\end{remark}

\section{Discussion}

The Epistemic Speed Limit (ESL) provides a unified perspective on learning
efficiency by reframing learning not in terms of attainable outcomes, but in
terms of the irreversible epistemic cost required to realize them within finite
time.
By separating epiplexity as an available epistemic resource from the entropy production
incurred during learning, the ESL establishes that any finite-time learning process necessarily incurs
irreversible epistemic entropy production.
As a consequence, learning outcomes can differ across learning procedures—not
because they access different information, but because they differ in how
efficiently they convert available epistemic structure into useful potential improvement.

\subsection{Learning Efficiency and Irreversible Cost}

From the perspective of the ESL, learning corresponds to a trajectory in the
space of probability distributions over model configurations.
Different learning procedures induce different trajectories, which may incur
substantially different amounts of irreversible entropy production even when they
achieve comparable final performance.

Learning procedures that follow geometrically inefficient trajectories incur
large entropy production in stabilizing internal representations.
In contrast, procedures that follow smoother or more direct trajectories can
realize comparable learning outcomes at lower irreversible cost.
Differences in learning efficiency therefore reflect differences in entropy production
rather than differences in accessible information or epistemic resources.

While the Epistemic Speed Limit constrains the irreversible cost of learning trajectories, improved transport efficiency
does not necessarily imply improved generalization performance.
In some regimes, solutions that are far from initialization may incur higher transport cost yet exhibit superior
generalization.
The ESL therefore constrains the cost of realizing ensemble transformations, rather than prescribing which endpoints
are desirable from a performance perspective.

Importantly, the ESL does not propose new learning algorithms.
Rather, it provides a principled explanation for why improvements in learning
often arise from reducing unnecessary irreversible cost, rather than from
accessing additional information.

\subsection{Curriculum Learning, Distillation, and Teacher Guidance}

Curriculum learning, knowledge distillation, and teacher guidance admit a
natural interpretation within the ESL framework.
These methods do not increase epiplexity or introduce new information.
Instead, they reshape learning trajectories so as to avoid unnecessary
irreversible entropy production.

Curriculum learning guides the learner along smoother paths in distribution
space, reducing abrupt shifts that would otherwise require large entropy
production.
Similarly, knowledge distillation and teacher guidance constrain the learning
trajectory, preserving epistemic flexibility and reducing entropy production during
learning.
Their empirical effectiveness follows from their ability to guide learning
closer to quasi-reversible trajectories.

\subsection*{Learning Speed and Quasi-Static Limits}

The finite-time form of the Epistemic Speed Limit also clarifies the role of
learning speed.
As the available training time $\mathcal T$ increases, the minimal entropy
production required to realize a given ensemble transformation decreases,
vanishing in the quasi-static limit $\mathcal T\to\infty$.
In this idealized regime, learning can in principle proceed in an
approximately reversible manner.

The ESL therefore does not prohibit low-dissipation learning.
Rather, it quantifies the fundamental trade-off between learning speed and
irreversible cost under finite-time constraints.
In practical settings, where training time, computational resources, and
environmental stability are limited, learning necessarily incurs nonzero
entropy production.
Excess dissipation then reflects geometrically inefficient learning
trajectories rather than unavoidable thermodynamic cost.

\subsection{Continual Learning and Irreversible Commitment}

The perspective offered by the Epistemic Speed Limit also clarifies why
subsequent learning after convergence is often difficult.
Once learning has progressed to a terminal ensemble, the distribution over
model configurations is typically highly concentrated, corresponding to a
low-entropy state.
This concentration reflects an irreversible epistemic commitment: probability
mass has been transported away from large regions of configuration space, and
those regions become effectively inaccessible under finite-time constraints.

Importantly, what is lost through learning is not information in the
information-theoretic sense, but reachability.
A low-entropy ensemble does not imply that alternative solutions do not exist in
principle, but rather that reaching them from the current ensemble requires
large probability transport.
Under the Epistemic Speed Limit, such transport necessarily incurs substantial
irreversible entropy production when attempted within finite time.

From this viewpoint, the difficulty of re-learning or adapting to new objectives
is not primarily a consequence of insufficient data or model capacity.
Instead, it arises because the ensemble has already undergone a geometrically
costly redistribution.
If the configurations that would reduce a new objective lie far outside the
support of the current ensemble, the Wasserstein distance between the current
and desired distributions is large, and rapid adaptation becomes intrinsically
expensive.

Conversely, ensembles with higher entropy retain broader support in parameter
space, making them closer, in a transport-geometric sense, to a wider range of
potential future objectives.
Such ensembles admit lower-cost adaptation, not because they encode more
information, but because they preserve reachability.
This observation highlights that entropy should not be viewed as something to be
maximized or minimized in isolation.
Rather, entropy mediates a trade-off between objective-driven improvement and
future adaptability.

The Epistemic Speed Limit thus reframes the stability--plasticity dilemma as a
geometric trade-off.
Excessive concentration leads to efficient short-term optimization but incurs
large transport costs for subsequent learning, while excessive dispersion slows
objective-driven progress.
Effective learning procedures implicitly balance these competing demands by
shaping learning trajectories that reduce unnecessary irreversible cost while
maintaining sufficient distributional flexibility for future adaptation.

\subsection{Implications for Intelligence Growth}

The following discussion should be understood as a conceptual implication rather than a concrete bound on any specific
learning system.
Here, ``intelligence growth'' refers abstractly to the accumulation of epistemic structure over finite time, rather than
to any particular operational metric or task performance.

Beyond specific learning procedures, the ESL implies a fundamental constraint on
the growth of intelligence over finite time.
Even if epiplexity is unbounded in principle—through increasing data, model
capacity, or representational richness—the realization of epistemic structure
necessarily incurs irreversible entropy production.

This does not preclude rapid improvements in learning performance.
Rather, it constrains the efficiency with which such improvements can be
realized.
Claims of instantaneous or arbitrarily explosive intelligence growth implicitly
assume vanishing entropy production, infinite learning time, or a breakdown of the
assumptions underlying irreversible learning dynamics.

Viewed through this lens, intelligence is shaped not only by the availability of
epistemic structure, but by the history of irreversible epistemic commitments
made during learning.
The Epistemic Speed Limit thus highlights finite-time efficiency and
irreversibility as central organizing principles of learning and intelligence.

\section{Conclusion}

In this work, we have argued that learning should be understood not only in
terms of what epistemic structure is available in principle, but in terms of the
irreversible cost required to realize such structure through finite-time
dynamics.
While epiplexity characterizes the epistemic structure accessible under
computational and representational constraints, actual learning necessarily
involves irreversible commitment to specific model configurations, accompanied
by unavoidable entropy production.

To make this distinction explicit, we introduced a free-energy formulation in
which learning is represented as a transport process over ensembles of model
configurations.
Within this framework, reductions in free energy can be algebraically decomposed into changes in the
ensemble-averaged objective and changes in entropy.
Irreversible entropy production, by contrast, is a path-dependent quantity
that accounts for the free-energy reduction realized along finite-time learning
trajectories.
This decomposition highlights a central fact: meaningful learning over finite
time cannot be realized without incurring an irreversible epistemic cost.

Our main theoretical result, the Epistemic Speed Limit (ESL), formalizes this
observation as a finite-time law.
The ESL does not bound learning outcomes or restrict what can be learned in
principle.
Instead, it lower-bounds the irreversible entropy production required to realize
a prescribed ensemble transformation within a given time horizon.
This bound is geometric in nature, determined by the transport cost between
initial and final ensemble distributions, and holds independently of the
specific learning algorithm employed.

The Epistemic Speed Limit also clarifies why further learning after apparent
convergence is often intrinsically difficult.
Once a learning process has completed, the ensemble over model configurations is
typically highly concentrated, corresponding to a low-entropy state.
This concentration reflects an irreversible epistemic commitment: probability
mass has been transported away from large regions of configuration space, which
thereafter become effectively unreachable under finite-time constraints.

Crucially, what is lost through learning is not information in the
information-theoretic sense, but reachability.
Alternative configurations that would support different objectives may still
exist in principle, yet adapting to them from a concentrated ensemble requires
large distributional transport.
Under the Epistemic Speed Limit, such transport necessarily incurs substantial
irreversible cost.
From this perspective, the difficulty of continual learning and re-adaptation
arises not from insufficient data or model capacity, but from the irreversible
history of epistemic commitments accumulated during prior learning.

From this perspective, differences between learning procedures arise not because
they access different epistemic resources, but because they manage unavoidable
irreversible cost with differing efficiency.
Procedures such as curriculum learning, knowledge distillation, and teacher
guidance do not circumvent information-theoretic limits.
Rather, they reshape learning trajectories so as to reduce unnecessary entropy
production, allowing a larger fraction of available epistemic structure to be
converted into realized potential improvement.

More broadly, the ESL implies a fundamental constraint on learning and
intelligence over finite time.
Even if epistemic structure is abundant or unbounded in principle, its
realization necessarily incurs irreversible entropy production.
Claims of instantaneous or arbitrarily explosive intelligence growth therefore
implicitly assume vanishing dissipation, infinite learning time, or a breakdown
of the assumptions underlying irreversible learning dynamics.

Viewed in this way, learning and intelligence are governed not solely by the
availability of information or representational capacity, but by finite-time
efficiency constraints imposed by irreversible epistemic commitment.
The Epistemic Speed Limit thus identifies entropy production and transport
geometry as central organizing principles of learning dynamics, complementing
classical information-theoretic accounts with a fundamentally dynamical
perspective.

\bibliographystyle{plain}
\bibliography{references}


\appendix
\section{Monotonicity of the Epistemic Free Energy}
\label{app:free_energy_monotonicity}

In this appendix, we provide a complete derivation showing that the epistemic
free-energy functional decreases monotonically along the Fokker--Planck learning
dynamics and that its rate of decrease coincides with the epistemic entropy
production rate defined in the main text.

\begin{lemma}[Free-energy dissipation along Fokker--Planck dynamics]
\label{lem:free_energy_dissipation}
Let $q_s(\theta)$ be a smooth probability density evolving on $s\in[0,1]$
according to the Fokker--Planck equation
\begin{equation}
\partial_s q_s
=
\nabla\cdot\!\left(q_s \nabla \Phi(\theta)\right)
+
T\,\Delta q_s,
\qquad T>0,
\label{eq:fp_appendix}
\end{equation}
where $\Phi(\theta)$ is a time-independent objective function.
Define the epistemic free-energy functional
\begin{equation}
\mathcal F[q]
:=
\int q(\theta)\,\Phi(\theta)\,d\theta
-
T\,H[q],
\qquad
H[q]:=-\int q(\theta)\log q(\theta)\,d\theta.
\label{eq:free_energy_appendix}
\end{equation}
Assume that $q_s$ is strictly positive and decays sufficiently fast at infinity
(or satisfies appropriate boundary conditions) so that all integrations by parts
are justified.
Then the free energy satisfies
\begin{equation}
\frac{d}{ds}\mathcal F[q_s]
=
-\sigma_s
\le 0,
\label{eq:free_energy_decay}
\end{equation}
where
\begin{equation}
\sigma_s
:=
\int q_s(\theta)\,\|v_s(\theta)\|^2\,d\theta,
\qquad
v_s(\theta):=-\nabla\!\left(\Phi(\theta)+T\log q_s(\theta)\right),
\label{eq:entropy_production_appendix}
\end{equation}
is the epistemic entropy production rate.
Consequently,
\begin{equation}
\label{eq:free_energy_integrated}
\mathcal F[q_0]-\mathcal F[q_1]
=
\int_0^1 \sigma_s\,ds.
\end{equation}
\end{lemma}
\begin{proof}
We differentiate the free-energy functional along the trajectory $q_s$.
Using $\frac{d}{ds}\int q\Phi=\int \Phi\,\partial_s q$ and
$\frac{d}{ds}\int q\log q=\int (1+\log q)\,\partial_s q$, we obtain
\begin{equation}
\frac{d}{ds}\mathcal F[q_s]
=
\int \bigl(\Phi + T(1+\log q_s)\bigr)\,\partial_s q_s\,d\theta.
\label{eq:dFds_step1}
\end{equation}
Substituting the Fokker--Planck equation~\eqref{eq:fp_appendix} yields
\begin{align}
\frac{d}{ds}\mathcal F[q_s]
&=
\int \bigl(\Phi + T(1+\log q_s)\bigr)
\Bigl[
\nabla\cdot(q_s\nabla\Phi)
+
T\,\Delta q_s
\Bigr]\,d\theta.
\label{eq:dFds_step2}
\end{align}
We treat the two terms separately.
By integration by parts and the assumed boundary conditions,
\begin{align}
\int \bigl(\Phi + T(1+\log q_s)\bigr)\,\nabla\cdot(q_s\nabla\Phi)\,d\theta
&=
-\int q_s\,(\nabla\Phi+T\nabla\log q_s)\cdot\nabla\Phi\,d\theta,
\label{eq:dFds_term1}
\\
T\int \bigl(\Phi + T(1+\log q_s)\bigr)\,\Delta q_s\,d\theta
&=
-T\int (\nabla\Phi+T\nabla\log q_s)\cdot\nabla q_s\,d\theta.
\label{eq:dFds_term2}
\end{align}
Using $\nabla q_s=q_s\nabla\log q_s$ and combining the two expressions, we obtain
\begin{equation}
\frac{d}{ds}\mathcal F[q_s]
=
-\int q_s(\theta)\,
\|\nabla\Phi(\theta)+T\nabla\log q_s(\theta)\|^2\,d\theta.
\label{eq:dFds_step3}
\end{equation}
Introducing the velocity field
$v_s(\theta):=-\nabla(\Phi(\theta)+T\log q_s(\theta))$,
this expression reduces to
\begin{equation}
\frac{d}{ds}\mathcal F[q_s]
=
-\int q_s(\theta)\,\|v_s(\theta)\|^2\,d\theta
=
-\sigma_s,
\end{equation}
which proves~\eqref{eq:free_energy_decay}.
Integrating over $s\in[0,1]$ yields~\eqref{eq:free_energy_integrated}.
\end{proof}

\end{document}